\documentclass{article}

\newcommand{\eg}{{\it e.g.}}
\newcommand{\ie}{{\it i.e.}}

\newcommand{\BA}{\begin{array}}
\newcommand{\EA}{\end{array}}

\newcommand{\BIT}{\begin{itemize}}
\newcommand{\EIT}{\end{itemize}}

\newcommand{\nats}{{\mathbb{N}}} 
\newcommand{\reals}{{\mathbb{R}}} 

\newcommand{\diag}{\mathop{\bf diag}}
\newcommand{\Expect}{\mathop{\bf E{}}}

\newcommand{\var}{\mathop{\bf var}}

\newcommand{\argmax}{\mathop{\rm argmax}}

\usepackage{url}
\usepackage{graphicx}

\usepackage[accepted]{icml2018}
\usepackage{amsfonts}
\usepackage{amsmath}
\usepackage{subcaption}
\usepackage{amsthm}
\usepackage{algpseudocode}

\newcommand{\U}{\mathcal{U}}
\newcommand{\T}{\mathcal{T}}
\newcommand{\X}{\mathcal{S}}
\newcommand{\A}{\mathcal{A}}
\newcommand{\Nc}{\mathcal{N}}
\newcommand{\Fc}{\mathcal{F}}

\newcommand{\vart}{{\textstyle{\bf var}_t}}
\newcommand{\Expectt}{{\textstyle{\mathop{\bf E{}}_t}}}

\newtheorem*{theorem*}{Theorem}
\newtheorem{theorem}{Theorem}

\newtheorem{lemma}{Lemma}

\newtheorem{assumption}{Assumption}

\begin{document}

\twocolumn[
\icmltitle{The Uncertainty Bellman Equation and Exploration}

\title{The Uncertainty Bellman Equation and Exploration}

\begin{icmlauthorlist}
\icmlauthor{Brendan O'Donoghue}{dm}
\icmlauthor{Ian Osband}{dm}
\icmlauthor{Remi Munos}{dm}
\icmlauthor{Volodymyr Mnih}{dm}
\end{icmlauthorlist}
\icmlaffiliation{dm}{DeepMind}
\icmlcorrespondingauthor{Brendan O'Donoghue}{bodonoghue@google.com}
\vspace{0.3in}
]
\printAffiliationsAndNotice{}

\begin{abstract}
We consider the exploration/exploitation problem in reinforcement learning.  For
exploitation, it is well known that the Bellman equation connects the value at
any time-step to the expected value at subsequent time-steps.  In this paper we
consider a similar \textit{uncertainty} Bellman equation (UBE), which connects
the uncertainty at any time-step to the expected uncertainties at subsequent
time-steps, thereby extending the potential exploratory benefit of a policy
beyond individual time-steps.  We prove that the unique fixed point of the UBE
yields an upper bound on the variance of the posterior distribution of
the Q-values induced by any policy. This bound can be much tighter than
traditional count-based bonuses that compound standard deviation rather than
variance.  Importantly, and unlike several existing approaches to optimism, this
method scales naturally to large systems with complex generalization.
Substituting our UBE-exploration strategy for $\epsilon$-greedy improves DQN
performance on 51 out of 57 games in the Atari suite.
\end{abstract}

\section{Introduction}

We consider the reinforcement learning (RL) problem of an agent interacting with
its environment to maximize cumulative rewards over time \cite{sutton:book}.
We model the environment as a Markov decision process (MDP), but where the agent
is initially uncertain of the true dynamics and mean rewards of the MDP
\cite{bellman,bertsekas2005dynamic}.  At each time-step, the agent performs an
action, receives a reward, and moves to the next state; from these data it can
learn which actions lead to higher payoffs.  This leads to the
\textit{exploration versus exploitation} trade-off: Should the agent investigate
poorly understood states and actions to improve future performance or instead
take actions that maximize rewards given its current knowledge?

Separating estimation and control in RL via `greedy' algorithms can lead to
premature and suboptimal exploitation.  To offset this, the majority of
practical implementations introduce some random noise or \textit{dithering} into
their action selection (such as $\epsilon$-greedy).  These algorithms will
eventually explore every reachable state and action infinitely often, but can
take exponentially long to learn the optimal policy \cite{kakade2003sample}.  By
contrast, for any set of prior beliefs the optimal exploration policy can be
computed directly by dynamic programming in the Bayesian belief space.  However,
this approach can be computationally intractable for even very small problems
\cite{guez2012efficient} while direct computational approximations can fail
spectacularly badly \cite{munos2014bandits}.

For this reason, most provably-efficient approaches to reinforcement learning
rely upon the \textit{optimism in the face of uncertainty} (OFU) principle
\cite{lai1985asymptotically,kearns2002near,brafman2002r}.  These algorithms give
a bonus to poorly-understood states and actions and subsequently follow the
policy that is optimal for this augmented optimistic MDP\@.  This optimism
incentivizes exploration but, as the agent learns more about the environment,
the scale of the bonus should decrease and the agent's performance should
approach optimality.  At a high level these approaches to OFU-RL build up
confidence sets that contain the true MDP with high probability
\cite{strehl2004exploration,lattimore2012pac,jaksch2010near}.  These techniques
can provide performance guarantees that are `near-optimal' in terms of the
problem parameters. However, apart from the simple `multi-armed bandit' setting
with only one state, there is still a significant gap between the upper and
lower bounds for these algorithms
\cite{lattimore2016regret,jaksch2010near,osband2016lower}.

One inefficiency in these algorithms is that, although the concentration may be
tight at each state and action independently, the combination of simultaneously
optimistic estimates may result in an extremely over-optimistic estimate for the
MDP as a whole \cite{osband2016posterior}.  Other works have suggested that a
Bayesian posterior sampling approach may not suffer from these inefficiencies
and can lead to performance improvements over OFU methods
\cite{strens2000bayesian,osband2013more,grande2014sample}.  In this paper we
explore a related approach that harnesses the simple relationship of the
uncertainty Bellman equation (UBE), where we define \emph{uncertainty} to be the
variance of the Bayesian posterior of the Q-values of a policy conditioned on
the data the agent has collected, in a sense similar to the
parametric variance of \citet{mannor2007bias}. Intuitively
speaking, if the agent has high uncertainty (as measured by high posterior
variance) in a region of the state-space then it should explore there, in order
to get a better estimate of those Q-values.  We show that, just as the Bellman
equation relates the value of a policy beyond a single time-step, so too does
the uncertainty Bellman equation propagate uncertainty values over multiple
time-steps, thereby facilitating `deep exploration' \cite{osband2017deep,
moerland2017expl}.

The benefit of our approach (which \textit{learns} the solution to the UBE and
uses this to guide exploration) is that we can harness the existing machinery
for deep reinforcement learning with minimal change to existing network
architectures.  The resulting algorithm shares a connection to the existing
literature of both OFU and intrinsic motivation
\cite{singh2004intrinsically,schmidhuber2009driven, white2010interval}.  Recent
work has further connected these approaches through the notion of `pseudo-count'
\cite{bellemare2016unifying,ostrovski2017count}, a generalization of the number
of visits to a state and action.  Rather than adding a pseudo-count based bonus
to the rewards, our work builds upon the idea that the more fundamental quantity
is the uncertainty of the value function and that naively
compounding count-based bonuses may lead to inefficient confidence sets
\cite{osband2016posterior}. The key difference is that the UBE compounds the
variances at each step, rather than standard deviation.

The observation that the higher moments of a value function also satisfy a form
of Bellman equation is not new and has been observed by some of the early papers
on the subject \cite{sobel1982variance}.  Unlike most prior work, we focus upon
the \textit{epistemic} uncertainty over the value function, as captured by the
Bayesian posterior, \ie, the uncertainty due to estimating a parameter using a
finite amount of data, rather than the higher moments of the reward-to-go
\cite{lattimore2012pac,azar2012sample,mannor2011mean,
bellemare2017distributional}.  For application to rich environments with complex
generalization we will use a deep learning architecture to \textit{learn} a
solution to the UBE, in the style of \cite{tamar2016learning}.

\section{Problem formulation}

We consider a finite horizon, finite state and action space MDP, with horizon
length $H \in \nats$, state space $\X$, action space $\A$ and rewards at time
period $h$ denoted by $r^h \in \reals$.  A policy $\pi = (\pi^1, \ldots, \pi^H)$
is a sequence of functions where each $\pi^h : \X \times \A \rightarrow
\reals_+$ is a mapping from state-action pair to the probability of taking that
action at that state, \ie, $\pi^h_{sa}$ is the probability of taking action $a$
at state $s$ at time-step $h$ and $\sum_{a} \pi^h_{sa} = 1$ for all $s \in \X$.
At each time-step $h$ the agent receives a state $s^h$ and a reward $r^h$ and
selects an action $a^h$ from the policy $\pi^h$, and the agent moves to the next
state $s^{h+1}$, which is sampled with probability $P^h_{s^{h+1} s^h a^h}$,
where $P^h_{s^\prime s a}$ is the probability of transitioning from state $s$ to
state $s^\prime$ after taking action $a$ at time-step $h$.  The goal of the
agent is to maximize the expected total discounted return $J$ under its policy
$\pi$, where $J(\pi) = \Expect \left[ \sum_{h=1}^H r^h \mid \pi \right]$.  Here
the expectation is with respect to the initial state distribution, the
state-transition probabilities, the rewards, and the policy $\pi$.

The action-value, or Q-value, at time step $l$ of a particular state under
policy $\pi$ is the expected total return from taking that action at that state
and following $\pi$ thereafter, \ie, $Q^l_{sa} = \Expect \left[ \sum_{h=l}^H r^h
\mid s^l = s, a^l = a, \pi \right]$ (we suppress the dependence on $\pi$ in this
notation). The value of state $s$ under policy $\pi$ at time-step $h$, $V^h(s) =
\Expect_{a \sim \pi^h_s} Q^h_{sa}$, is the expected total discounted return of
policy $\pi$ from state $s$.

The Bellman operator $\T^h$ for policy $\pi$ at each time-step $h$ relates the
value at each time-step to the value at subsequent time-steps via dynamic
programming
\cite{bellman},
\begin{equation}
\label{e-bellman_1}
    \T^h Q^{h+1}_{sa} = \mu^h_{sa} + \sum_{s^\prime, a^\prime} \pi^h_{s^\prime
    a^\prime} P^h_{s^\prime s a} Q^{h+1}_{s^\prime a^\prime}
\end{equation}
for all $(s, a)$, where $\mu = \Expect r$ is the mean reward.  The Q-values are
the unique fixed point of equation (\ref{e-bellman_1}), \ie, the solution to
$\T^h Q^{h+1} = Q^h$ for $h=1,\ldots, H$, where $Q^{H+1}$ is defined to be zero.
Several reinforcement learning algorithms have been designed around minimizing
the residual of equation (\ref{e-bellman_1}) to propagate knowledge of immediate
rewards to long term value \cite{sutton1988learning,watkins1989learning}.  In
the next section we examine a similar relationship for propagating the
\textit{uncertainties} of the Q-values, we call this relationship the
uncertainty Bellman equation.

\section{The uncertainty Bellman equation}
\label{s-ube}

In this section we derive a Bellman-style relationship that propagates the
uncertainty (variance) of the Bayesian posterior distribution over Q-values
across multiple time-steps.  Propagating the potential value of exploration over
many time-steps, or \textit{deep exploration}, is important for statistically
efficient RL \cite{kearns2002near,osband2017deep}.  Our main result, which we
state in Theorem~\ref{thm:ube}, is based upon nothing more than the dynamic
programming recursion in equation (\ref{e-bellman_1}) and some crude upper
bounds of several intermediate terms.  We will show that even in very simple
settings this approach can result in well-calibrated uncertainty estimates where
common count-based bonuses are inefficient \cite{osband2016posterior}.

\subsection{Posterior variance estimation}
We consider the Bayesian case, where we have priors over the mean reward $\mu$
and the transition probability matrix $P$ which we denote by $\phi_\mu$ and
$\phi_P$ respectively. We collect some data generated by
the policy $\pi$ and use it to derive posterior distributions over $\mu$ and
$P$, given the data.  We denote by $\Fc_t$ the sigma-algebra generated by all
the history up to episode $t$ (\eg, all the rewards, actions, and state
transitions for all episodes), and let the posteriors over the mean reward and
transition probabilities be denoted by $\phi_{\mu | \Fc_t}$ and $\phi_{P |
\Fc_t}$ respectively. If we sample $\hat \mu \sim \phi_{\mu | \Fc_t}$ and $\hat
P \sim \phi_{P | \Fc_t}$, then the resulting Q-values that satisfy
\[
\hat Q^h_{sa} = \hat \mu^h_{sa} + \sum_{s^\prime, a^\prime} \pi^h_{s^\prime
a^\prime} \hat P^h_{s^\prime s a} \hat Q^{h+1}_{s^\prime a^\prime},
\quad h=1,\ldots,H,
\]
where $\hat Q^{H+1} = 0$,
are a sample from the implicit posterior over Q-values, conditioned on the
history $\Fc_t$ \cite{strens2000bayesian}. In this section we compute a bound on
the variance (uncertainty) of the random variable $\hat Q$.
For the analysis we will require some additional assumptions.
\begin{assumption}
\label{ass-1}
The MDP is a directed acyclic graph.
\end{assumption}
This assumption means that the agent cannot revisit a state within the same
episode, and is a common assumption in the literature
\cite{osband2014generalization}. Note that any finite horizon MDP that doesn't
satisfy this assumption can be converted into one that does by `unrolling' the
MDP so that each state $s$ is replaced by $H$ copies of the state, one for each
step in the episode.
\begin{assumption}
\label{ass-2}
The mean rewards are bounded in a known interval, \ie,  $\mu^h_{sa} \in
[-R_\mathrm{max}, R_\mathrm{max}]$ for all $(s,a)$.
\end{assumption}
This assumption means we can bound the absolute value of the Q-values as
$|Q^h_{sa}| \leq Q_\mathrm{max} = H R_\mathrm{max}$. We will use this quantity
in the bound we derive below. This brings us to our first lemma.
\begin{lemma}
\label{l-bell-ineq}
For any random variable $x$ let
\[
\vart x = \Expect((x - \Expect(x | \Fc_t))^2 | \Fc_t)
\]
denote the variance of $x$ conditioned on $\Fc_t$.
Under the assumptions listed above, the variance of the Q-values under the
posterior satisfies the Bellman inequality
\[
\vart \hat Q^h_{sa}\leq \nu^h_{sa} + \sum_{s^\prime, a^\prime}
\pi^h_{s^\prime a^\prime} \Expect(\hat P^h_{s^\prime s a} | \Fc_t) \vart \hat
Q^{h+1}_{s^\prime a^\prime}
\]
for all $(s,a)$ and $h=1, \ldots, H$, where $\vart \hat Q^{H+1} = 0$
and where we call $\nu^h_{sa}$ the \emph{local uncertainty} at $(s, a)$, and it
is given by
\[
\begin{array}{l}
\nu^h_{sa} = \vart \hat \mu^h_{sa} +
Q^2_\mathrm{max} \sum_{s^\prime}
\vart \hat P^h_{s^\prime s a} /  \Expect(\hat P^h_{s^\prime s a} | \Fc_t).
\end{array}
\]
\begin{proof}
See the appendix.
\end{proof}
\end{lemma}
We refer to $\nu$ in the above lemma as the \emph{local}
uncertainty since it depends only on locally available quantities, and so can be
calculated (in principle) at each state-action independently.
Note that even though $\Expect(\hat P^h_{s^\prime s a} | \Fc_t)$ appears in the
denominator above, the local uncertainty is bounded, since for any random
variable $X$ on $(0,1]$ we have
\[
\var X / \Expect X \leq 1 - \Expect X \leq 1,
\]
and if for any $s^\prime$ we have that $\Expect(\hat P^h_{s^\prime s a} |
\Fc_t)=0$ Markov's inequality implies that $P^h_{s^\prime s a}=0$ and so we can
just remove that term from the sum since it contributes no uncertainty.

With this lemma we are ready to prove our main theorem.
\begin{theorem}[Solution of the uncertainty Bellman equation]
\label{thm:ube}
Under assumptions~\ref{ass-1} and~\ref{ass-2}, for any policy $\pi$ there exists
a unique $u$ that satisfies the uncertainty Bellman equation
\begin{equation}
\label{e-ube}
u^h_{sa} = \nu^h_{sa} + \sum_{s^\prime, a^\prime} \pi^h_{s^\prime a^\prime}
  \Expect(P^h_{s^\prime s a}|\Fc_t) u^{h+1}_{s^\prime a^\prime}
\end{equation}
for all $(s,a)$ and $h=1,\ldots,H$, where $u^{H+1} = 0$, and furthermore $u \geq
\vart \hat Q$ pointwise.
\end{theorem}

\begin{proof}
Let $\U^h$ be the Bellman operator that defines the uncertainty Bellman
equation, \ie, rewrite equation~(\ref{e-ube}) as
\[
u^h = \U^h u^{h+1},
\]
then to prove the result we use two essential properties of the Bellman
operator for a fixed policy. Firstly, the solution to the Bellman equation exists
and is unique, and secondly the Bellman operator is monotonically
non-decreasing in its argument, \ie, if $x \geq y$ pointwise then $\U^h x \geq
\U^h y$ pointwise \cite{bertsekas2005dynamic}.  The proof proceeds by induction;
assume that for some $h$ we have $\vart \hat Q^{h+1} \leq u^{h+1}$, then we have
\[
  \vart \hat Q^h \leq \U^h \vart \hat Q^{h+1} \leq \U^h u^{h+1} = u^h,
\]
where we have used the fact that the variance satisfies the Bellman
\emph{inequality} from lemma~\ref{l-bell-ineq}, and the base case holds because
$\vart \hat Q^{H+1} = u^{H+1} = 0$.
\end{proof}

We conclude with a brief discussion on why the variance of the posterior is
useful for exploration.  If we had access to the true posterior distribution
over the Q-values then we could take actions that lead to states with higher
uncertainty by, for example, using Thompson sampling
\cite{thompson1933likelihood, strens2000bayesian}, or constructing Q-values that
are high probability upper bounds on the true Q-values and using the OFU
principle \cite{kaufmann2012bayesian}. However, calculating the true posterior
is intractable for all but very small problems. Due to this difficulty prior
work has sought to approximate the posterior distribution \cite{osband2017deep},
and use that to drive exploration.  In that spirit we develop another
approximation of the posterior, in this case it is motivated by the Bayesian
central limit theorem which states that, under some mild conditions, the
posterior distribution converges to a Gaussian as the amount of data increases
\cite{berger2013statistical}. With that in mind, rather than computing the full
posterior we approximate it as $\Nc(\bar Q, \diag(u))$ where $u$ is the
solution to the uncertainty Bellman equation~(\ref{e-ube}), and consequently is
a guaranteed upper bound on the true variance of the posterior, and $\bar Q$
denotes the mean Q-values under the posterior at episode $t$, \ie, the unique
solution to
\[
\bar Q^h_{sa} = \Expect(\hat \mu^h_{sa} |\Fc_t) +
\sum_{s^\prime, a^\prime} \pi^h_{s^\prime a^\prime} \Expect(\hat P^h_{s^\prime s
a}|\Fc_t) \bar Q^{h+1}_{s^\prime a^\prime},
\]
for $h=1,\ldots,H$, and $\bar Q^{H+1} = 0$.
With this approximate posterior we can perform Thompson sampling as an
exploration heuristic. Specifically, at state $s$ and time-step $h$ we select
the action using
\begin{equation}
\label{e-thomp}
a = \argmax_b(\bar Q^h_{sb} + \zeta_b (u^h_{sb})^{1/2})
\end{equation}
where $\zeta_b$ is sampled from $\mathcal{N} (0, 1)$. Our goal is for the agent
to explore states and actions where it has higher uncertainty.  This is in
contrast to the commonly used $\epsilon$-greedy \cite{mnih-atari-2013} and
Boltzmann exploration strategies \cite{mnih2016asynchronous,o2016pgq,haarnoja17}
which simply inject noise into the agents actions. We shall see in the
experiments that our strategy can dramatically outperform these naive
heuristics.

\subsection{Comparison to traditional exploration bonus}
Consider a simple decision problem with known deterministic transitions, unknown
rewards, and two actions at a root node, as depicted in
Figure~\ref{f-tab-mdp-cartoon}.  The first action leads to a single reward $r_1$
sampled from $\Nc(\mu_1, \sigma^2)$ at which point the episode terminates, and
the second action leads to an chain of length $H$ consisting of states each
having random reward $r_2$ independently sampled from $\Nc(\mu_2 / H, \sigma^2/
H)$.

Take the case where each action at the root has been taken $n$ times and where
the uncertainty over the rewards at each state concentrates like $1/n$ (\eg,
when the prior is an improper Gaussian). In this case the \emph{true}
uncertainty about the value of each action is identical and given by $\sigma^2 /
n$. This is also the answer we get from the uncertainty Bellman equation, since
for action 1 we obtain $u_1 = \sigma^2 / n$ (since $\var_t P = 0$) and for
action 2 the uncertainty about the reward at each state along the chain is given
by $\sigma^2 / H n$ and so we have $u_2 = \sum_{h=1}^H \sigma^2 / H n = \sigma^2
/n$.

Rather than considering the variance of the value as a whole, the majority of
existing approaches to OFU provide exploration bonuses at each state and action
independently and then combine these estimates via union bound.  In this
context, even a state of the art algorithm such as UCRL2 \cite{jaksch2010near}
would augment the rewards at each state with a bonus proportional to the
standard deviation of the reward estimate at each state
\cite{bellemare2016unifying}.  For the first action this would be ${\rm
ExpBonus}_1 = \sigma / \sqrt{n}$, but for the second action this would
be accumulated along the chain to be
\[
  {\rm ExpBonus}_2 = \sum_{h=1}^H \frac{\sigma}{{\sqrt{Hn}}} = \sigma
  \sqrt{\frac{H}{n}}
\]
In other words, the bonus afforded to the second action is a factor of
$\sqrt{H}$ larger than the true uncertainty. The agent would have to take the
second action a factor of $H$ more times than the first action in order to have
the same effective bonus given to each one. If the first action was actually
superior in terms of expected reward, it would take the agent far longer to
discover that than an agent using the correct uncertainties to select actions.
The essential issue is that, unlike the variance, the standard deviations do not
obey a Bellman-style relationship.

In Figure~\ref{f-tab-mdp} we show the results of an experiment showing this
phenomenon.  Action 1 had expected reward $\mu_1 = 1$, and action 2 had expected
reward $\mu_2 = 0$. We set $\sigma = 1$ and $H = 10$, and the results are
averaged over $500$ seeds. We compare two agents, one using the uncertainty
Bellman equation to drive exploration and the other
agent using a count-based reward bonus. Both agents take actions and use the
results to update their beliefs about the value of each action. The agent using
the UBE takes the action yielded by Thompson sampling as in equation
(\ref{e-thomp}).  The exploration-bonus agent takes the action $i$ that
maximizes $\hat Q_i + \beta \log(t) {\rm ExpBonus}_i$ (the $\log(t)$
term is required to achieve a regret bound \cite{jaksch2010near}, but doesn't
materially affect the previous argument) where $\beta > 0$ is a hyper-parameter
chosen by a sweep and where $\hat Q_i$ is the
estimate of the value of action $i$.  Figure~\ref{f-tab-mdp} shows the
\emph{regret} of each agent vs number of episodes. Regret measures how
sub-optimal the rewards the agent has received so far are, relative to the
(unknown) optimal policy, and lower regret is better \cite{cesa2006prediction}.

The agent using the uncertainty Bellman equation has well calibrated
uncertainty estimates and consequently quickly figures out that the first action
is better. By contrast, the exploration bonus agent takes significantly longer
to determine that the first action is better due to the fact that the bonus
afforded to the second action is too large, and consequently it suffers
significantly higher regret.

\begin{figure}
\begin{center}
\includegraphics[width=0.435\textwidth]{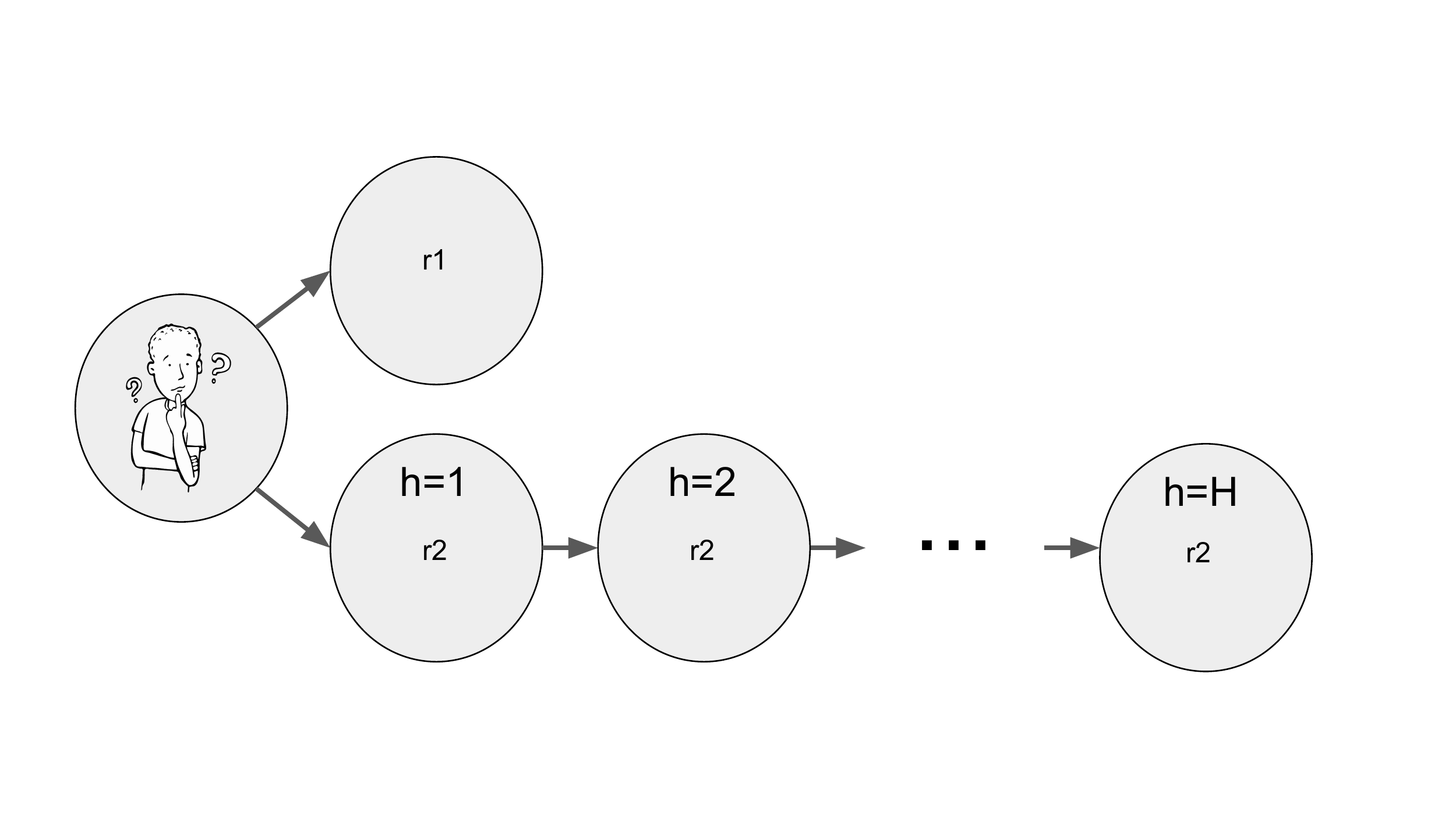}
\caption{Simple tabular MDP.}
\label{f-tab-mdp-cartoon}
\end{center}
\end{figure}

\begin{figure}
\begin{center}
\includegraphics[width=0.4\textwidth]{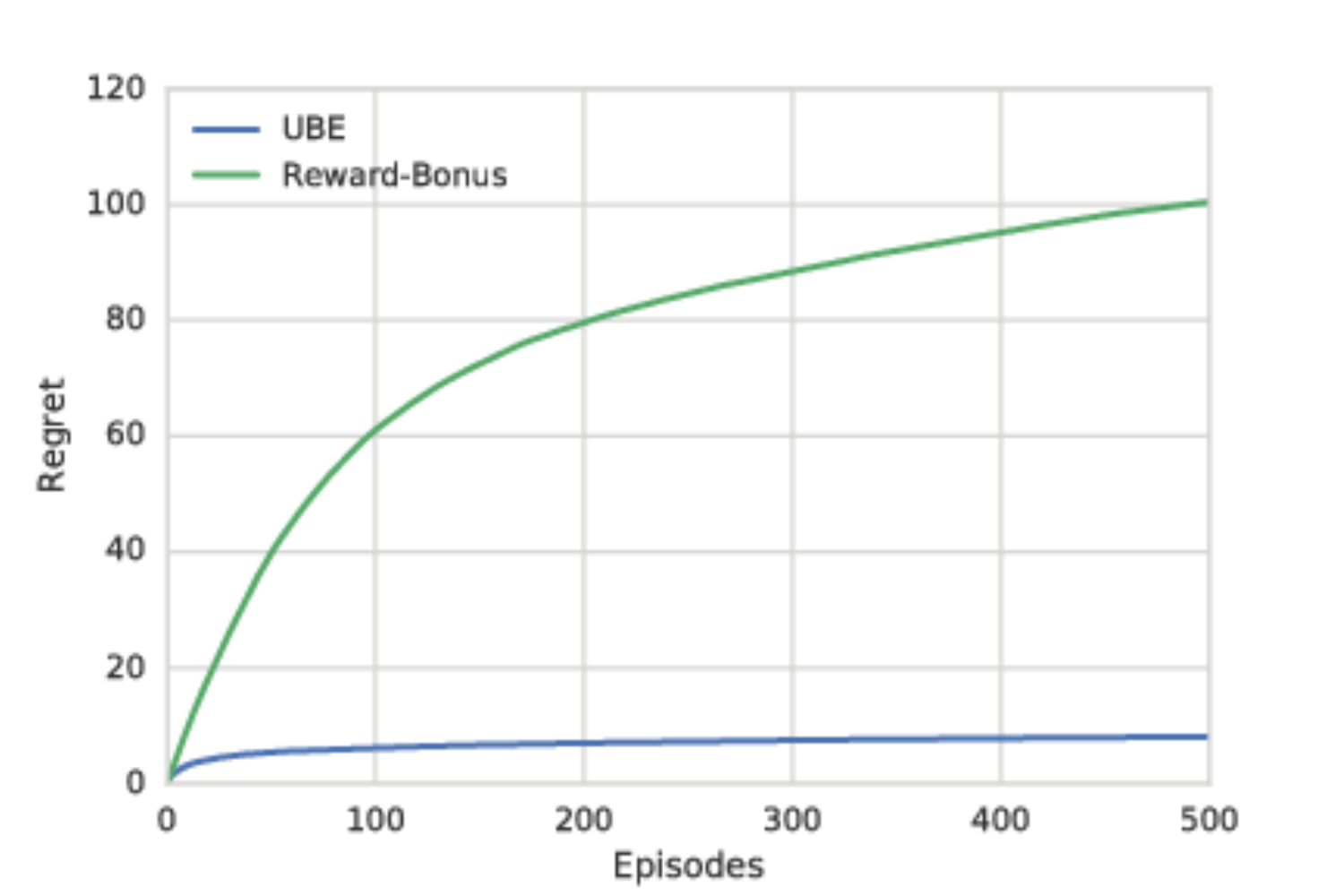}
\caption{Regret over time for the simple tabular MDP.}
\label{f-tab-mdp}
\end{center}
\end{figure}

\section{Estimating the local uncertainty}
\label{s-lin-approx}

Section~\ref{s-ube} outlined how the uncertainty Bellman equation can be used to
propagate local estimates of the variance of $\hat Q$ to global estimates
for the uncertainty. In this section we present some pragmatic approaches to
estimating the local uncertainty $\nu$ that we can then use for practical
learning algorithms inspired by Theorem~\ref{thm:ube}.  We do not claim that
these approaches are the only approaches to estimating the local
uncertainty, or even that these simple approximations are in any sense the
`best'.  Investigating these choices is an important area of future research,
but outside the scope of this short paper.  We present a simple progression from
tabular representations, to linear function approximation and then to non-linear
neural network architectures.

\paragraph{Tabular value estimate.}
Consider the case where the posterior over the mean rewards concentrates
at least as fast the reciprocal of the visit count, \ie,
\[
\vart \hat \mu^h_{sa} \leq \sigma_r^2 / n^h_{sa}
\]
where $\sigma_r$ is the variance of the reward process and $n^h_{sa}$ is the
\emph{visit count} of the agent to state $s$ and action $a$ at time-step $h$,
up to episode $t$.  This is the case when, for example, the rewards and the
prior over the mean reward are both Gaussian. Furthermore, if we assume that the
prior over the transition function is Dirichlet then it is straightforward to
show that
\[
\sum_{s^\prime} \vart \hat P^h_{s^\prime s a}  / \Expect(\hat P^h_{s^\prime s a} |
\Fc_t) \leq |S| / n^h_{sa}
\]
where $|S|$ is the number of next states reachable from $s, a$. This holds
since the likelihood of the transition function is a categorical distribution,
which is conjugate to the Dirichlet distribution and the variance of a Dirichlet
concentrates like the reciprocal of the sum of the counts of each category.
Under these assumptions we can bound the local uncertainty as
\[
\nu^h_{sa} \leq (\sigma_r^2 + Q_\mathrm{max}^2|S|) / n^h_{sa}.
\]
In other words, the local uncertainty can be modeled under these assumptions as
a constant divided by the visit count.

\paragraph{Linear value estimate.}
In the non-tabular case we need some way to estimate the inverse counts in order
to approximate the local uncertainty. Consider a linear value function estimator
$\hat Q^h_{sa} = \phi(s)^T w_a$ for each state and action with fixed basis
functions $\phi(s) : \X
\rightarrow \reals^D$ and learned weights $w_a \in \reals^D$, one for each action.
This setting allows for some generalization between states and actions through
the basis functions. For any fixed dataset we can find the
least squares solution for each action $a$ \cite{boyan1999least},
\[
  \begin{array}{ll}
    \mbox{minimize}_{w_a} \sum_{i=1}^N (\phi(s_i)^T w_a - y_i)_2^2,
  \end{array}
\]
where each $y_i \in \reals$ is a regression target (\eg, a Monte Carlo return
from that state-action).  The solution to this problem is $w_a^\star = (\Phi_a^T
\Phi_a)^{-1}\Phi_a^T y$, where $\Phi_a$ is the matrix consisting of the
$\phi(s_i)$ vectors stacked row-wise (we use the subscript $a$ to denote the
fact that action $a$ was taken at these states).  We can compute the variance
of this estimator, which will provide a proxy for the inverse counts
\cite{bellemare2016unifying}. If we model the targets $y_i$ as IID with unit
variance, then $\var_t w_a^\star = (\Phi_a^T
\Phi_a)^{-1}$.  Given a new state vector $\phi_s$, the variance of the Q-value
estimate at $(s,a)$ is then $\vart\phi_s^T w_a^\star = \phi_s^T(\Phi_a^T
\Phi_a)^{-1}\phi_s$, which we can take to be our estimate of the inverse counts,
\ie, set $(\hat n^h_{sa})^{-1} = \phi_s^T(\Phi_a^T \Phi_a)^{-1}\phi_s$. Now we can
estimate the local uncertainty as
\begin{equation}
\label{e-lin-var}
\hat \nu^h_{sa} = \beta^2 (\hat n^h_{sa})^{-1} = \beta^2 \phi_s^T(\Phi_a^T
\Phi_a)^{-1}\phi_s \end{equation}
for some $\beta$, which in the tabular case (\ie, where $\phi(s) = e_{s}$
and $D = |\X|$) is equal to $\beta^2 / n^h_{sa}$, as expected.

An agent using this notion of uncertainty must maintain and update the
matrix $\Sigma_a = (\Phi_a^T \Phi_a)^{-1}$ as it receives new data. Given new
sample $\phi$, the updated matrix $\Sigma_a^+$ is given by
\begin{equation}
\label{e-smw}
\begin{array}{rcl}
\Sigma_a^+ &=& \left( \begin{bmatrix} \Phi_a \\
\phi^T \end{bmatrix}^T \begin{bmatrix} \Phi_a \\
\phi^T \end{bmatrix} \right)^{-1} = (\Phi_a^T\Phi_a + \phi\phi^T)^{-1}\\&=&
\Sigma_a - (\Sigma_a\phi\phi^T\Sigma_a) / (1+\phi^T\Sigma_a\phi)
\end{array}
\end{equation}
by the Sherman-Morrison-Woodbury formula \cite{golub2012matrix},
the cost of this update is one matrix multiply and one matrix-matrix
subtraction per step.

\paragraph{Neural networks value estimate.}

If we are approximating our Q-value function using a neural network then
the above analysis does not hold. However if the last layer of the network is
linear, then the Q-values are approximated as $Q^h_{sa} = \phi(s)^T w_a$,
where $w_a$ are the weights of the last layer associated with action $a$ and
$\phi(s)$ is the output of the network up to the last layer for state $s$. In
other words we can think of a neural network as learning a useful set of basis
functions such that a linear combination of them approximates the Q-values.
Then, if we ignore the uncertainty in the $\phi$ mapping, we can reuse the
analysis for the purely linear case to derive an \emph{approximate} measure of
local uncertainty that might be useful in practice.

This scheme has some advantages. As the agent progresses it is learning a state
representation that helps it achieve the goal of maximizing the return.  The
agent will learn to pay attention to small but important details (\eg, the ball
in Atari `breakout') and learn to ignore large but irrelevant changes (\eg, if
the background suddenly changes).  This is a desirable property from the point
of view of using these features to drive exploration, because the states that
differ only in irrelevant ways will be aliased to (roughly) the same state
representation, and states that differ is small but important ways will be
mapped to quite different state vectors, permitting a more task-relevant measure
of uncertainty.

\section{Deep Reinforcement Learning}

Previously we proved that under certain conditions we can bound the variance of
the posterior distribution of the Q-values, and we used the resulting
uncertainty values to derive an exploration strategy.  Here we discuss the
application of that strategy to deep-RL\@. In this case several of the assumptions
we have made to derive theorem~\ref{thm:ube} are violated. This puts us firmly
in the territory of heuristic.  Specifically, the MDPs we apply this to will not
be directed acyclic graphs, the policy that we are estimating the uncertainty
over will not be fixed, we cannot exactly compute the local uncertainty, and we
won't be solving the UBE exactly.  However, empirically, we demonstrate that
this heuristic can perform well in practice, despite the underlying assumptions
being violated.

Our strategy involves \emph{learning} the uncertainty estimates, and then using
them to sample Q-values from the approximate posterior, as in equation
(\ref{e-thomp}).  The technique is described in pseudo-code in
Algorithm~\ref{a-ubealg}. We refer to the technique as `one-step' since the
uncertainty values are updated using a one-step SARSA Bellman backup, but it is
easily extendable to the $n$-step case.  The algorithm takes as input a neural
network which has two output `heads', one which is attempting to learn the
optimal Q-values as normal, the other is attempting to learn the uncertainty
values of the current policy (which is constantly changing). We do not allow the
gradients from the uncertainty output head to flow into the trunk of the
network; this ensures the Q-value estimates are not perturbed by the changing
uncertainty signal.  For the local uncertainty measure we use the linear basis
approximation described in section~\ref{s-lin-approx}. Algorithm~\ref{a-ubealg}
incorporates a discount factor $\gamma \in (0,1)$, since deep RL often uses a
discount even in the purely episodic case.  In this case the Q-learning update
uses a $\gamma$ discount and the Uncertainty Bellman equation~(\ref{e-ube}) is
augmented with a $\gamma^2$ discount factor.

\begin{algorithm}[t]
\caption{One-step UBE exploration with linear uncertainty estimates.}
\begin{algorithmic}
\Require Neural network outputting $Q$ and $u$ estimates, Q-value learning
subroutine \verb|qlearn|, Thompson sampling hyper-parameter $\beta>0$
\State Initialize $\Sigma_a = \mu I$ for each $a$, where $\mu > 0$
\State Get initial state $s$, take initial action $a$
\For{episode $t=1,\ldots, $}
\For{time-step $h=2,\ldots, H+1$}
\State Retrieve $\phi(s)$ from input to last network layer
\State Receive new state $s^\prime$ and reward $r$
\State Calculate $\hat Q^h_{s^\prime b}$ and $u^h_{s^\prime b}$ for each action $b$
\State Sample $\zeta_b \sim \mathcal{N}(0,1)$ for each $b$ and calculate
\[
    a^\prime = \argmax_b(\hat Q^h_{s^\prime b} + \beta \zeta_b (u^h_{s^\prime b})^{1/2})
\]
\State Calculate
\[
  y =
    \left\{
    \begin{array}{l}
      \phi(s)^T\Sigma_a\phi(s),\ \text{if}\ h=H+1\\
      \phi(s)^T\Sigma_a\phi(s)+ \gamma^2 u^{h}_{s^\prime a^\prime},\
      \text{o.w.}
    \end{array} \right.
  \]
\State Take gradient step with respect to error
\[(y - u^{h-1}_{sa})^2\]
\State Update Q-values using \verb|qlearn|$(s,a,r,s^\prime, a^\prime)$
\State Update $\Sigma_{a}$ according to eq. (\ref{e-smw})
\State Take action $a^\prime$
\State Set $s \leftarrow s^\prime$, $a \leftarrow a^\prime$
\EndFor
\EndFor
\end{algorithmic}
\label{a-ubealg}
\end{algorithm}

\subsection{Experimental results}

Here we present results of Algorithm (\ref{a-ubealg}) on the Atari suite of
games \cite{bellemare-ale}, where the network is attempting to learn the
Q-values as in DQN \cite{mnih-atari-2013,mnih-dqn-2015} and the uncertainties
simultaneously. The only change to vanilla DQN we made was to replace the
$\epsilon$-greedy policy with Thompson sampling over the learned uncertainty
values, where the $\beta$ constant in (\ref{e-thomp}) was chosen to be $0.01$
for all games, by a parameter sweep. We used the exact same network
architecture, learning rate, optimizer, pre-processing and replay scheme as
described in \citet{mnih-dqn-2015}.  For the uncertainty sub-network we used a
single fully connected hidden layer with 512 hidden units followed by the output
layer. We trained the uncertainty head using a separate RMSProp optimizer
\cite{tieleman2012lecture} with learning rate $10^{-3}$. The addition of the
uncertainty head and the computation associated with it, only reduced the
frame-rate compared to vanilla DQN by about 10\% on a GPU, so the additional
computational cost of the approach is negligible.

We compare two versions of our approach: a $1$-step method and an $n$-step
method where we set $n$ to $150$.  The $n$-step method accumulates the
uncertainty signal over $n$ time-steps before performing an update which should
lead to the uncertainty signal propagating to earlier encountered states faster,
at the expense of increased variance of the signal. Note that in all cases
the Q-learning update is always $1$-step; our $n$-step implementation
only affects the uncertainty update.

We compare our approaches to vanilla DQN, and also to an exploration bonus
intrinsic motivation approach, where the agent receives an augmented reward
consisting of the extrinsic reward and the square root of the linear uncertainty
in equation (\ref{e-lin-var}), which was scaled by a hyper-parameter chosen to
be $0.1$ by a sweep.  In this case a stochastic policy was still required for
good performance and so we used $\epsilon$-greedy with the DQN annealing
schedule.

We trained all strategies for 200M frames (about 8 days on a GPU).  Each game
and strategy was tested three times per method with the same hyper-parameters
but with different random seeds, and all plots and scores correspond to an
average over the seeds. All scores were normalized by subtracting the average
score achieved by an agent that takes actions uniformly at random. Every 1M
frames the agents were saved and evaluated (without learning) on 0.5M frames,
where each episode is started from the random start condition described in
\cite{mnih-dqn-2015}. The final scores presented correspond to first averaging
the evalution score in each period across seeds, then taking the max average
episodic score observed during any evalution period. Of the tested strategies
the $n$-step UBE approach was the highest performer in 32 out of 57 games, the
$1$-step UBE approach in 14 games, DQN in 1 game, the exploration bonus strategy
in 7 games, and there were 3 ties. In Table~\ref{t-atari-means-basic} we give
the mean and median normalized scores as percentage of an expert human
normalized score across all games, and the number of games where the agent is
`super-human', for each tested algorithm. Note that the mean scores are
significantly affected by a single outlier with very high score (`Atlantis'),
and therefore the median score is a better indicator of agent performance. In
Figure~\ref{f-super-human} we plot the number of games at super-human
performance against frames for each method, and in Figure~\ref{f-median} we plot
the median performance across all games versus frames, where a
score of $1.0$ denotes human performance.  The results across all 57 games, as
well as the learning curves for all 57 games, are given in the appendix.

Of particular interest is the game `Montezuma's Revenge', a notoriously
difficult exploration game where no one-step algorithm has managed to learn
anything useful.  Our $1$-step strategy learns in 200M frames a policy that is
able to consistently get about 500 points, which is the score the agent gets for
picking up the first key and moving into the second room.  In
Figure~\ref{f-montezuma} we show the learning progress of the agents for 500M
frames where we set the Thompson sampling parameter slightly higher; $0.016$
instead of $0.01$ (since this game is a challenging exploration task it stands
to reason that a higher exploration parameter is required).  By the end of 500M
frames the $n$-step agent is consistently getting around 3000 points, which is
several rooms of progress.  These scores are close to state-of-the-art, and are
state-of-the-art for one-step methods (like DQN) to the best of our knowledge.

In the recent work by \citet{bellemare2016unifying}, and
the follow-up work by \citet{ostrovski2017count}, the authors add an intrinsic
motivation signal to a DQN-style agent that has been modified to use the full
Monte Carlo return of the episode when learning the Q-values. Using Monte Carlo
returns dramatically improves the performance of DQN in a way unrelated to
exploration, and due to that change we cannot compare the numerical results
directly. In order to have a point of comparison we implemented our own
intrinisic motivation exploration signal, as discussed above.  Similarly, we
cannot compare directly to the numerical results obtained by Bootstrap DQN
\cite{osband2016deep} since that agent is using Double-DQN, a variant of DQN
that achieves a higher performance in a way unrelated to exploration. However,
we note that our approach achieves a higher evaluation score in 27 out of the 48
games tested in the Bootstrap DQN paper despite using an inferior base DQN
implementation, and it runs at a significantly lower computational and memory
cost.

\begin{table}[h]
\begin{center}
\small
\begin{tabular}{ c| c c c}
& \bf{mean} & \bf{median} & \bf{$>$ human}\\
\hline
DQN & 688.60 & 79.41 & 21 \\
DQN Intrinsic Motivation & 472.93 & 76.73 & 24 \\
DQN UBE 1-step & 776.40 & 94.54 & 26 \\
DQN UBE n-step & 439.88 & 126.41 & 35 \\
\end{tabular}
\end{center}
\caption{Scores for the Atari suite, as a percentage of human score.}
\label{t-atari-means-basic}
\end{table}

\begin{figure}
\begin{center}
\includegraphics[width=0.5\textwidth]{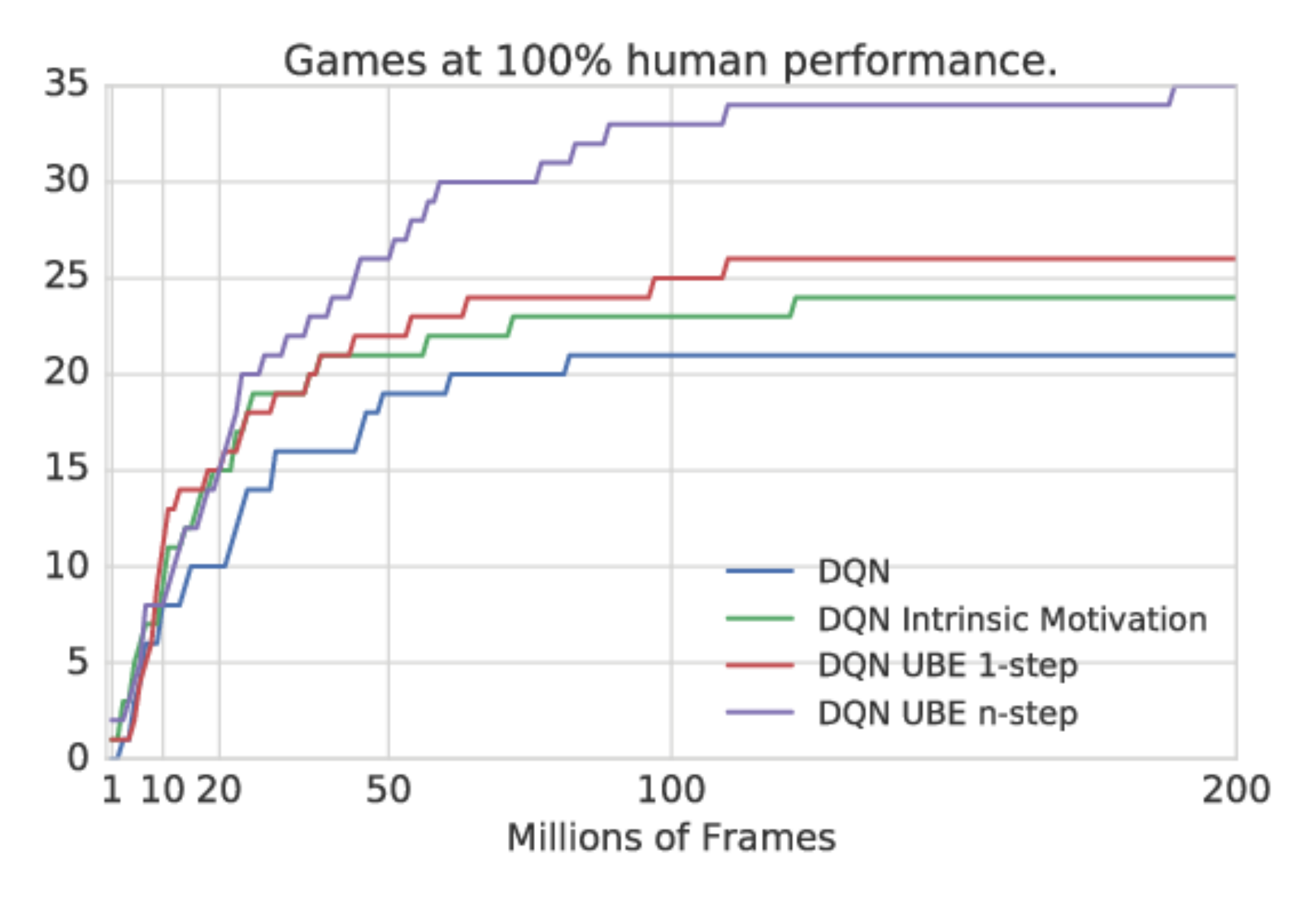}
\caption{Number of games at super-human performance.}
\label{f-super-human}
\end{center}
\end{figure}

\begin{figure}
\begin{center}
\includegraphics[width=0.5\textwidth]{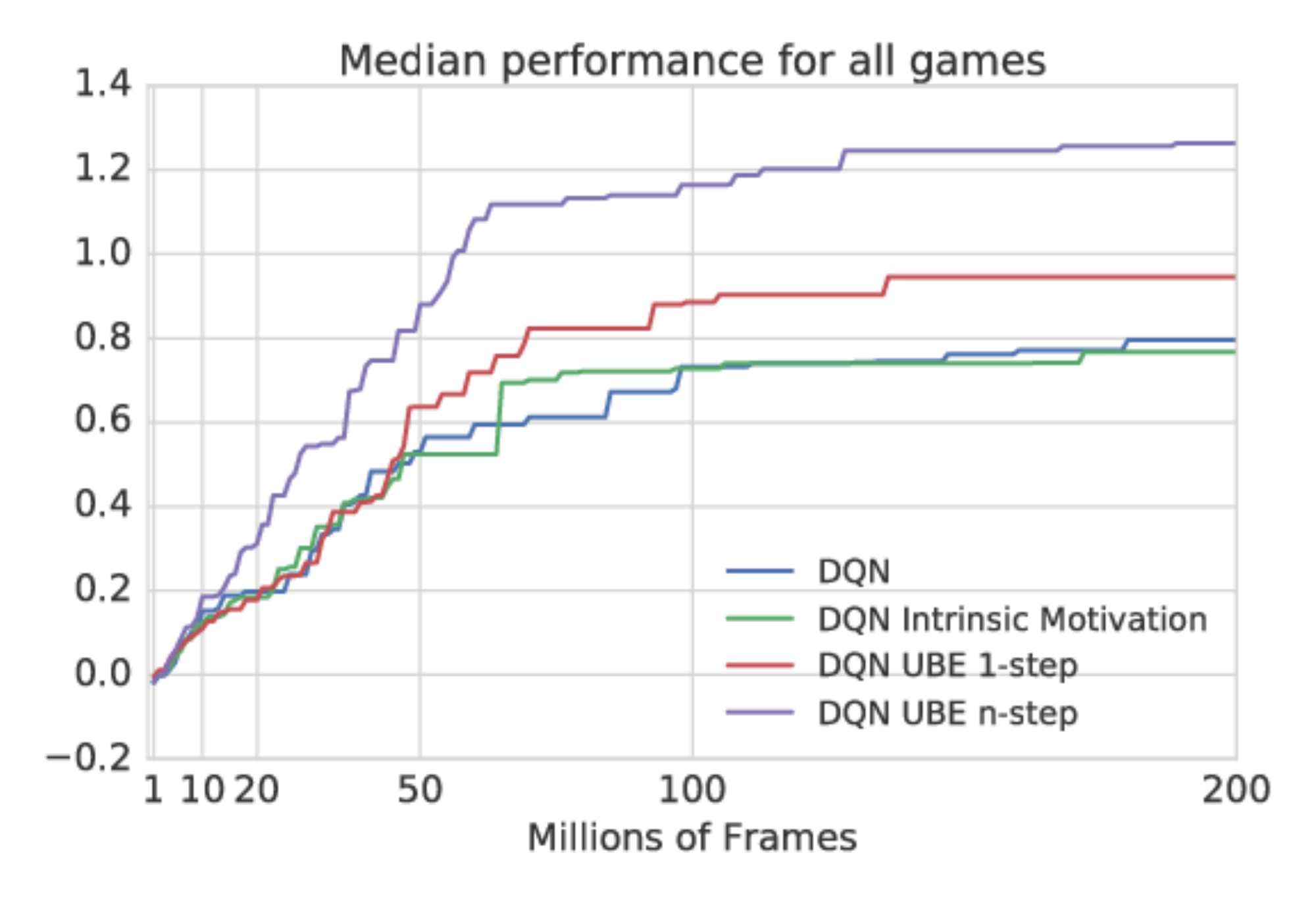}
\caption{Normalized median performance across all games, a score of $1.0$ is
human-level performance.}
\label{f-median}
\end{center}
\end{figure}

\begin{figure}
\begin{center}
\includegraphics[width=0.5\textwidth]{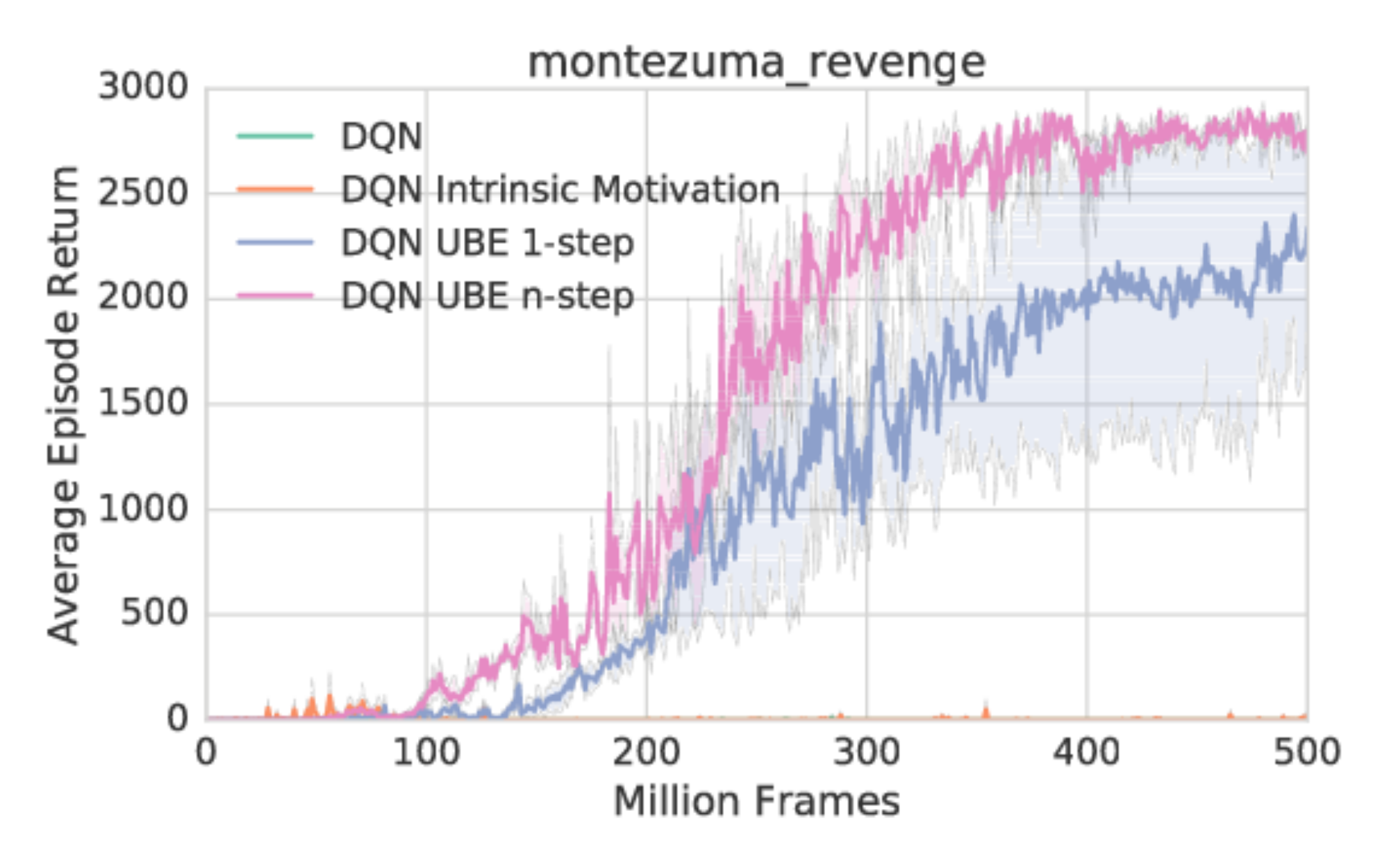}
\caption{Montezuma's Revenge performance.}
\label{f-montezuma}
\end{center}
\end{figure}

\section{Conclusion}
In this paper we derived a Bellman equation for the uncertainty over the
Q-values of a policy. This allows an agent to propagate uncertainty across many
time-steps in the same way that value propagates through time in the standard
dynamic programming recursion.  This uncertainty can be used by the agent to
make decisions about which states and actions to explore, in order to gather
more data about the environment and learn a better policy.  Since the
uncertainty satisfies a Bellman recursion, the agent can learn it using the same
reinforcement learning machinery that has been developed for value functions.
We showed that a heuristic algorithm based on this learned uncertainty can boost
the performance of standard deep-RL techniques. Our technique was able to
significantly improve the performance of DQN across the Atari suite of games,
when compared against naive strategies like $\epsilon$-greedy.

\section{Acknowledgments}
We thank Marc Bellemare, David Silver, Koray Kavukcuoglu, Daniel Kasenberg, and
Mohammad Gheshlaghi Azar for useful discussion and suggestions on the paper.

\bibliography{ube}
\bibliographystyle{icml2018}

\clearpage
\onecolumn
\section*{Appendix}
\subsection*{Proof of Lemma 1}
For ease of exposition in this derivation we shall use the notation
$\Expectt(\cdot)$ to denote the expectation of a random variable conditioned on
the history $\Fc_t$, rather than the usual $\Expect(\cdot | \Fc_t)$.

Recall that if we sample $\hat \mu$ from the posterior over the mean rewards
$\phi_{\mu| \Fc_t}$, and $\hat P$ from the posterior over the transition
probability matrix $\phi_{P| \Fc_t}$ then the Q-values that are the unique
solution to
\[
\hat Q^h_{sa} = \hat \mu^h_{sa} + \sum_{s^\prime, a^\prime} \pi^h_{s^\prime a^\prime}
\hat P^h_{s^\prime s a} \hat Q^{h+1}_{s^\prime a^\prime}
\]
are a sample from the (implicit) posterior over Q-values, conditioned on
$\Fc_t$ \cite{strens2000bayesian}.

Using the definition of the conditional variance
\[
\begin{array}{rcl}
\vart \hat Q^h_{sa}&=&\Expectt \Big(\hat Q^h_{sa} - \Expectt \hat Q^h_{sa}
 \Big)^2 \\

&=&\Expectt \Big(\hat \mu^h_{sa} - \Expectt \hat \mu^h_{sa} + \sum_{s^\prime,
a^\prime}
\pi^h_{s^\prime a^\prime} (\hat
P^h_{s^\prime s a}\hat Q^{h+1}_{s^\prime a^\prime} - \Expectt(\hat P^h_{s^\prime s a}
\hat Q^{h+1}_{s^\prime a^\prime}))  \Big)^2 \\

&=&\Expectt \Big(\hat \mu^h_{sa} - \Expectt \hat \mu^h_{sa}  \Big)^2+
\Expectt\Big(\sum_{s^\prime,
a^\prime}
\pi^h_{s^\prime a^\prime} (\hat
P^h_{s^\prime s a}\hat Q^{h+1}_{s^\prime a^\prime} - \Expectt(\hat P^h_{s^\prime s a}
\hat Q^{h+1}_{s^\prime a^\prime})) \Big)^2 \\


\end{array}
\]
where we have used the fact that $\hat \mu^h_{sa}$ is conditionally independent
(conditioned on $\Fc_t$) of $\hat P^h_{s^\prime s a}\hat Q^{h+1}_{s^\prime
a^\prime}$ because assumption~\ref{ass-1} implies that $\hat Q^{h+1}_{s^\prime
a^\prime}$ depends only on downstream quantities.

Now we make the assumption that $\Expectt \hat P^h_{s^\prime s a} > 0$ for all
$h, s^\prime, s, a$.  Note that this is not a restriction because if $\Expectt
\hat P^h_{s^\prime s a} = 0$, then the fact that $P$ is nonnegative combined
with Markov's inequality implies that $P^h_{s^\prime s a} = 0$, in which case we
can just remove that term from the sum, since it contributes no variance (or
equivalently $s^\prime$ is not reachable after taking $s,a$).  Furthermore, note
that any sample must satisfy $\sum_{s^\prime} \hat P^h_{s^\prime s a} = 1$ which
implies that $\sum_{s^\prime} \Expectt \hat P^h_{s^\prime s a} = 1$ and
therefore $\sum_{s^\prime a^\prime } \pi_{s^\prime a^\prime} \Expectt \hat
P^h_{s^\prime s a} = 1$. In other words $\pi^h_{s^\prime a^\prime}\Expectt \hat
P^h_{s^\prime s a}$ defines a probability distribution over $s^\prime,
a^\prime$.  With this we can bound the second term as
\[
\begin{array}{l}
\Expectt\Big(\sum_{s^\prime, a^\prime}
\pi^h_{s^\prime a^\prime} (\hat
P^h_{s^\prime s a}\hat Q^{h+1}_{s^\prime a^\prime} - \Expectt(\hat P^h_{s^\prime s a}
\hat Q^{h+1}_{s^\prime a^\prime})) \Big)^2\\

\qquad=\Expectt\Big(\sum_{s^\prime,
a^\prime}
\pi^h_{s^\prime a^\prime}(\Expectt \hat P^h_{s^\prime s a}/ \Expectt \hat P^h_{s^\prime s a}) (\hat
P^h_{s^\prime s a}\hat Q^{h+1}_{s^\prime a^\prime} - \Expectt(\hat P^h_{s^\prime s a}
\hat Q^{h+1}_{s^\prime a^\prime}))\Big)^2\\

\qquad\leq \sum_{s^\prime,
a^\prime} \pi^h_{s^\prime a^\prime}\Expectt \hat P^h_{s^\prime s a}\Expectt \Big(\hat
P^h_{s^\prime s a}\hat Q^{h+1}_{s^\prime a^\prime} - \Expectt(\hat P^h_{s^\prime s a}
\hat Q^{h+1}_{s^\prime a^\prime}) \Big)^2/ \Big(\Expectt \hat P^h_{s^\prime s a}\Big)^2

\end{array}
\]
by applying Jensen's inequality to the quadratic. Combining these we have
\begin{equation}
\label{e-app1}
\vart \hat Q^h_{sa} \leq \vart \hat \mu^h_{sa}  +
\sum_{s^\prime,
a^\prime}\pi^h_{s^\prime a^\prime}\Expectt \hat P^h_{s^\prime s a}\Expectt \Big(\hat
P^h_{s^\prime s a}\hat Q^{h+1}_{s^\prime a^\prime} - \Expectt(\hat P^h_{s^\prime s a}
\hat Q^{h+1}_{s^\prime a^\prime}) \Big)^2/ \Big(\Expectt \hat P^h_{s^\prime s a}\Big)^2.
\end{equation}
Assumption~\ref{ass-1} also implies that that $\hat
P^h_{s^\prime s a}$ and $\hat Q^{h+1}_{s^\prime a^\prime}$ are conditionally
independent, and so we can write the middle expectation in the second term as
\[
\begin{array}{rcl}
\Expectt \Big(\hat
P^h_{s^\prime s a}\hat Q^{h+1}_{s^\prime a^\prime} - \Expectt(\hat P^h_{s^\prime s a}
\hat Q^{h+1}_{s^\prime a^\prime})\Big)^2 &=&
\Expectt\Big((\hat P^h_{s^\prime s a}- \Expectt\hat P^h_{s^\prime s
a}) \hat Q^{h+1}_{s^\prime a^\prime}
+ (\Expectt\hat P^h_{s^\prime s a}) (\hat Q^{h+1}_{s^\prime a^\prime} - \Expectt\hat
Q^{h+1}_{s^\prime a^\prime})\Big)^2 \\
&=& \Expectt\Big( (\hat P^h_{s^\prime s a}- \Expectt\hat P^h_{s^\prime s
a}) \hat Q^{h+1}_{s^\prime a^\prime} \Big)^2 + \Expectt\Big((\Expectt\hat
P^h_{s^\prime s a}) (\hat Q^{h+1}_{s^\prime a^\prime} - \Expectt\hat
Q^{h+1}_{s^\prime a^\prime})\Big)^2.
\end{array}
\]
Now using the conditional independence property again and the fact that the
Q-values are bounded, as implied by assumption~\ref{ass-2}, we have
\[
\Expectt\Big( (\hat P^h_{s^\prime s a}- \Expectt\hat P^h_{s^\prime s
a}) \hat Q^{h+1}_{s^\prime a^\prime} \Big)^2 =
\Expectt\Big (\hat P^h_{s^\prime s a}- \Expectt\hat P^h_{s^\prime s
a}\Big)^2 \Expectt \Big(\hat Q^{h+1}_{s^\prime a^\prime} \Big)^2
\leq
Q^2_\mathrm{max} \vart \hat P^h_{s^\prime s a},
\]
and similarly 
\[
\Expectt\Big((\Expectt\hat P^h_{s^\prime s a}) (\hat Q^{h+1}_{s^\prime a^\prime} -
\Expectt\hat Q^{h+1}_{s^\prime a^\prime})\Big)^2 =(\Expectt\hat P^h_{s^\prime s a})^2
\Expectt\Big(\hat Q^{h+1}_{s^\prime a^\prime} - \Expectt\hat Q^h(s^\prime,
a^\prime))\Big)^2  = (\Expectt \hat P^h_{s^\prime s a})^2 \vart \hat
Q^{h+1}_{s^\prime a^\prime}.
\]
And so we have
\[
\Expectt \Big(\hat
P^h_{s^\prime s a}\hat Q^{h+1}_{s^\prime a^\prime} - \Expectt(\hat P^h_{s^\prime
s a} \hat Q^{h+1}_{s^\prime a^\prime})\Big)^2 \leq Q^2_\mathrm{max} \vart \hat
P^h_{s^\prime s a} +  (\Expectt \hat P^h_{s^\prime s a})^2 \vart \hat
Q^{h+1}_{s^\prime a^\prime}.
\]
Putting this together with equation (\ref{e-app1}) we obtain
\[
\vart \hat Q^h_{sa} \leq \nu^h_{sa} + \sum_{s^\prime, a^\prime}
\pi^h_{s^\prime a^\prime} \Expectt\hat P^h_{s^\prime s a} \vart \hat
Q^{h+1}_{s^\prime a^\prime}
\]
where $\nu^h_{sa}$ is the \emph{local} uncertainty, and is given by
\[
\nu^h_{sa} = \vart \hat \mu^h_{sa} +
Q^2_\mathrm{max} \sum_{s^\prime}
\vart \hat P^h_{s^\prime s a} / \Expectt\hat P^h_{s^\prime s a}.
\]

\newpage
\subsection*{Atari suite scores}
\begin{table}[H]
\begin{center}
\footnotesize
\begin{tabular}{ c|c|c|c|c}
Game  & DQN & DQN Intrinsic Motivation & DQN UBE 1-step & DQN UBE n-step \\
\hline
alien &      40.96 &      28.13 & \bf{46.90} &      43.61 \\
amidar &      58.17 &      41.50 & \bf{83.48} &      83.14 \\
assault &     479.34 &     647.16 &     887.18 & \bf{1112.28} \\
asterix &      67.26 &      74.01 &      82.34 & \bf{130.95} \\
asteroids &       1.86 &       2.25 & \bf{3.54} &       2.61 \\
atlantis & \bf{28662.58} &   14382.27 & \bf{28655.71} &    6889.77 \\
bank heist &      56.43 &      54.20 &      97.72 & \bf{162.90} \\
battle zone &      70.47 &      77.80 &      49.63 & \bf{101.68} \\
beam rider &      63.22 &      58.60 &     106.79 & \bf{149.93} \\
berzerk &      18.29 &      21.30 &      44.92 & \bf{2284.45} \\
bowling &      13.03 & \bf{20.10} &      -3.64 &      14.12 \\
boxing &     782.87 &     784.74 &     811.68 & \bf{816.42} \\
breakout &    1438.45 &    1377.36 & \bf{2045.13} &    1474.66 \\
centipede & \bf{34.54} &      25.04 &      22.62 &      18.94 \\
chopper command &      58.06 &      74.19 &      69.67 & \bf{75.90} \\
crazy climber &     428.59 &     431.10 &     496.21 & \bf{499.34} \\
defender &      83.94 &      78.20 &      94.54 & \bf{209.70} \\
demon attack &     338.43 &     372.29 &     765.46 & \bf{897.89} \\
double dunk &     481.10 &     575.90 &     750.00 & \bf{1031.25} \\
enduro &      92.01 &     102.36 &       9.94 & \bf{154.21} \\
fishing derby &      85.03 &      95.68 &      97.51 & \bf{105.80} \\
freeway &      81.08 &     121.85 &       0.02 & \bf{126.41} \\
frostbite &       9.86 &      16.02 &      11.26 & \bf{21.22} \\
gopher &     392.09 &     558.74 &     656.70 & \bf{867.45} \\
gravitar &       2.38 &       4.07 &       2.06 & \bf{4.91} \\
hero &      49.27 & \bf{63.34} &      47.63 &      34.58 \\
ice hockey &      64.17 & \bf{64.49} &      38.22 &      58.43 \\
jamesbond &     193.96 &     230.36 &     198.40 & \bf{430.89} \\
kangaroo &     266.39 &     311.21 &     202.79 & \bf{537.69} \\
krull &     656.32 &     702.15 & \bf{1033.61} &     838.45 \\
kung fu master &     103.20 &     112.79 &     128.30 & \bf{153.40} \\
montezuma revenge &      -0.49 &       4.21 & \bf{11.43} &       0.80 \\
ms pacman &      16.11 &      16.17 &      18.42 & \bf{19.82} \\
name this game &     114.24 &     101.23 & \bf{129.62} &     127.99 \\
phoenix &     115.17 &     157.34 & \bf{199.39} &     167.51 \\
pitfall & \bf{5.49} & \bf{5.49} & \bf{5.49} & \bf{5.49} \\
pong &     112.04 &     112.06 & \bf{116.32} & \bf{116.37} \\
private eye &      -0.04 & \bf{0.44} &      -0.44 &       0.33 \\
qbert &      79.41 &     103.53 &     124.96 & \bf{125.85} \\
riverraid &      43.41 &      46.33 &      60.73 & \bf{68.24} \\
road runner &     524.43 &     531.48 &     722.15 & \bf{732.09} \\
robotank &     770.41 &     779.94 &     414.29 & \bf{803.06} \\
seaquest &       7.93 & \bf{10.61} &       9.01 &       9.31 \\
skiing &      11.47 &      15.26 &      31.55 & \bf{54.31} \\
solaris &       3.48 &       8.23 & \bf{18.56} &      -6.53 \\
space invaders &      87.42 &      72.95 &     125.29 & \bf{138.65} \\
star gunner &     309.47 &     398.98 &     456.86 & \bf{547.39} \\
surround &      16.67 &      19.44 &      37.42 & \bf{60.98} \\
tennis &     145.58 &     145.58 & \bf{227.93} &     145.58 \\
time pilot &      92.53 &      76.73 &      88.36 & \bf{121.64} \\
tutankham &     148.29 & \bf{191.72} &     132.62 &     138.12 \\
up n down &      92.62 &      95.73 &     139.35 & \bf{142.12} \\
venture &       4.05 & \bf{10.77} &      -1.24 &       8.73 \\
video pinball &    1230.64 &    2393.54 & \bf{3354.58} &    1992.87 \\
wizard of wor &      73.26 &      73.96 & \bf{187.48} &     118.76 \\
yars revenge &      37.93 &      23.24 & \bf{46.21} &      44.36 \\
zaxxon &      35.24 &      53.14 & \bf{62.20} &      56.44 \\
\hline
\end{tabular}
\end{center}
\caption{Normalized scores for the Atari suite from random starts, as a percentage of human normalized score.}
\label{t-atari-basic}
\end{table}


\end{document}